\documentclass[11pt]{article}
\usepackage{booktabs,latexsym}
\usepackage{graphicx}
\usepackage{amsthm,amsmath,amssymb}
\usepackage{algorithm}

\newenvironment{myitemize}{\begin{list}{$\bullet$}{\setlength{\leftmargin}{0pt}
\setlength{\itemindent}{\labelwidth}}}
{\end{list}}

\newcommand{\raf}[1]{(\ref{#1})}

\newcommand{\OR}{\bigvee}
\newcommand{\AN}{\bigwedge}
\newcommand{\cM}{{\mathcal M}}
\newcommand{\cN}{{\mathcal N}}
\newcommand{\charset}{{\mathit char}}
\newcommand{\modd}{{\mathit mod}}

\newcommand{\ccap}{Cl_{\wedge}}
\newcommand{\At}{{\mathit At}}
\newcommand{\Lit}{{\mathit Lit}}
\newcommand{\ON}{{\mathit ON}}
\newcommand{\OFF}{{\mathit OFF}}
\newcommand{\ol}[1]{{\overline{#1}}}

\newcommand{\ga}{\alpha}
\newcommand{\gb}{\beta}

\newcommand{\gvp}{\varphi}

\newcommand{\df}{\stackrel{\mathrm{def}}{=}}

\newcommand{\nop}[1]{}

\newtheorem{thm}{Theorem}
\newtheorem{prop}[thm]{Proposition}
\newtheorem{exmp}[thm]{Example}
\newtheorem{lem}[thm]{Lemma}
\newtheorem{cor}[thm]{Corollary}

\title{Deductive Inference for the Interiors and Exteriors of Horn Theories\thanks{An extended abstract of this article was presented in
  \emph{Proceedings of Algorithms and Computation, 19th International
 Symposium} (ISAAC~2008), Lecture Notes in Computer Science, Vol.~5369, pp.~390--401,
  Springer-Verlag Berlin Heidelberg, 2008.}}
\author{
\begin{tabular}{cc}  
\begin{minipage}{6cm}
\begin{center}
Kazuhisa Makino\\
Department of Mathematical Informatics, \\
University of Tokyo, \\
Tokyo,  113-8656, Japan.\\
{\tt makino@mist.i.u-tokyo.ac.jp}
\end{center}
\end{minipage}
\begin{minipage}{6cm}
\begin{center}
Hirotaka Ono\\
Department of Computer Science and Communication
Engineering, \\ Kyushu University, \\ 
Fukuoka 812-8581, Japan.\\
{\tt ono@csce.kyushu-u.ac.jp}
\end{center}
\end{minipage}
\end{tabular}
}
\date{}

\begin{document}
\maketitle
\begin{abstract}
In this paper, we investigate the deductive inference 
for the interiors and exteriors of Horn knowledge bases, 
where the interiors and exteriors were introduced by Makino and Ibaraki
\cite{interior} to study stability properties of knowledge bases. 
We present a linear time algorithm for 
 the deduction for the interiors and show that it is co-NP-complete for the deduction for  the exteriors. 
Under model-based representation, we show that
 the deduction problem for interiors is NP-complete while the one for
 exteriors is co-NP-complete. 
 As for Horn envelopes of the exteriors,  
we show that it is linearly solvable under model-based representation, 
while it is co-NP-complete under formula-based representation. 
We also discuss the polynomially solvable cases 
for all the intractable problems. 
\end{abstract}

\section{Introduction}
Knowledge-based systems are commonly used to store the sentences 
as our knowledge for the
purpose of having automated reasoning such as deduction 
applied to them
(see e.g., \cite{brachman-levesque-04}). 
Deductive inference is a fundamental mode of reasoning, and 
usually abstracted as follows: 
Given the knowledge base $KB$, assumed to capture our knowledge about
the domain in question, and a query $\chi$ that is assumed to capture
the situation at hand, decide whether $KB$ implies $\chi$, denoted by
$KB \models \chi$, which can be understood  as the question: ``Is $\chi$
 consistent with the current state of knowledge ?''

In this paper, we consider the interiors and exteriors of knowledge base.
Formally, for a given positive integer
$\ga$, the $\ga$-interior of $KB$, denoted by $\sigma_{-\ga}(KB)$, 
is a knowledge that consists of the models (or assignments) $v$ 
satisfying that the $\ga$-neighbors of $v$ are
all models of $KB$, 
and  the $\ga$-exterior of $KB$, denoted by $\sigma_{\ga}(KB)$, 
is a knowledge that consists of the models $v$ satisfying that 
at least one of the $\ga$-neighbors of $v$ is a model of $KB$  \cite{interior}.
Intuitively, the interior consists of the models $v$ that 
{\em strongly} satisfy $KB$,
since all neighbors of $v$ are models of $KB$, 
while 
the exterior consists of the models $v$ that 
 {\em weakly} satisfy $KB$,  
since at least one of the $\ga$-neighbors of $v$ is a model of $KB$.
Here we note that $v$ might not satisfy $KB$, even if we say that it
weakly satisfies $KB$.
As mentioned in \cite{interior}, the interiors and exteriors of
knowledge base
 merit study in their own right, since they shed light on the structure
 of knowledge base. 
Moreover, let us consider the situation in which knowledge base $KB$ is {\em
not perfect} in the sense that some sentences in $KB$ are wrong and/or
some are missing in $KB$ (see also \cite{interior}). 

Suppose that we use $KB$ as a knowledge base for automated reasoning, say, 
deductive inference $KB \models \chi$. 
Since $KB$ does not represent {\em real} knowledge $KB^*$, 
 the reasoning result is no longer true.
However, if we use the interior  $\sigma_{-\ga}(KB)$ of $KB$ as a knowledge base
and have $\sigma_{-\ga}(KB) \not\models \chi$, then
we  can expect that the result is true for  real knowledge $KB^*$,
since $\sigma_{-\ga}(KB)$ consists of models which strongly satisfy $KB$.
On the other hand, if we use the exterior $\sigma_{\ga}(KB)$ of KB as a knowledge base
and have $\sigma_{\ga}(KB) \models \chi$, then
we  can expect that the result is true for real knowledge $KB^*$,
since $\sigma_\ga(KB)$ consists of models which weakly satisfy $KB$.
In this sense, the interiors and exteriors help to have {\em safe} reasoning.

\medskip

\noindent{\bf Main problems considered}. In this paper, we study 
the deductive inference for the interiors and exteriors of propositional
Horn theories, where Horn theories are ubiquitous in Computer Science, 
cf.\ \cite{mako-87}, and
are of particular relevance in Artificial Intelligence and Databases. 
\nop{%%%%%%%%
since they consist of clauses equivalent to rules of form $b_1\land
b_2\land\cdots\land b_m \rightarrow b_0$ and to constraints of form
$\neg( b_1\land b_2\land \cdots\land b_n)$, where all $b_i$'s are
atoms. }
It is known that important reasoning problems like deductive
inference and satisfiability checking, which are
intractable for arbitrary propositional theories, are solvable in
linear time for Horn theories (cf.\ \cite{linear}).

More precisely, we address the following problems:

\begin{myitemize}
\item Given a Horn theory $\Sigma$, a clause $c$, and nonnegative
      integer $\ga$, we consider the problems of deciding if deductive
      queries hold for the $\ga$-interior and exterior of $\Sigma$,
      i.e., $\sigma_{-\ga}(\Sigma) \models c$ and  $\sigma_{\ga}(\Sigma)
      \models c$.  It is well-known \cite{linear} that a deductive query for a Horn
      theory can be answered in linear time. Note that it is
      intractable to construct
      the interior and exterior for a Horn theory
      \cite{interior,posinterior}, and hence a direct method (i.e., 
      first construct the interior (or exterior) and then check a
      deductive query) is not possible efficiently.

\item We contrast traditional formula-based (syntactic) with
model-based (semantic) representation of Horn theories. 
The latter form of representation has been proposed as an alternative form of
representing and accessing a logical knowledge base, cf.\
\cite{dech-pear-92,EIM99,EM,kaut-etal-93,kaut-etal-95,kavv-etal-93,khar-roth-96,khar-roth-97}.
In model-based reasoning, $\Sigma$ is represented by a subset of its
models $\cM$, which are commonly called {\em characteristic
models}. 
As shown by Kautz {\em et al.}  \cite{kaut-etal-93}, 
the deductive inference can be done in polynomial time, 
given its characteristic models.  
\nop{
As shown by Kautz {\em et al.}  \cite{kaut-etal-93}, the deductive and
      abductive inferences can be done in polynomial time, if a Horn
      theory is given by its characteristic models, 
where we note that the abduction is intractable under
formula-based representation \cite{selm-leve-90,selm-leve-96}. Similar
results were shown for other theories by Khardon and Roth
\cite{khar-roth-96}.
}

\item Finally, we consider Horn approximations for the exteriors of
     Horn theories. Note that the interiors of Horn theories are Horn,
      while the exteriors might not be Horn. 
      We deal with the least upper bounds, called the {\em Horn
      envelopes} \cite{approximation},  for  the exteriors of Horn theories. 
\end{myitemize}

\noindent{\bf Main results}. 
%nontrivial to derive. 
We investigate the problems mentioned above from an algorithmical viewpoint.  
For all the problems, we provide either polynomial time algorithms or proofs of
the intractability;  thus, our work gives a complete
picture of the tractability/intractability frontier  of deduction for
interiors and exteriors of Horn theories. 
Our main results can be summarized as follows (see Figure \ref{fig-0}).

\begin{myitemize}
\item  We present a linear time algorithm for the deduction for the
       interiors of a given Horn theory, 
and show that  it is co-NP-complete for the deduction for  the exteriors. 
Thus, the positive result for ordinary deduction for Horn theories
       extends to the interiors, but does not to the exteriors.  
We also show that the deduction for  the exteriors is possible in
       polynomial time, if $\ga$ is bounded by a constant or if
       $|N(c)|$ is bounded by a logarithm of the input
       size, where $N(c)$ corresponds to the set of negative literals in
       $c$. 
 
\item  Under model-based representation, we show that
the consistency problem and the deduction for the interiors of Horn theories are both co-NP-complete.  
As for the exteriors, we show that the deduction is co-NP-complete. 
%This means that formula-based representation is better than the
% model-based representation, in the sense of the deduction for interiors. 
We also show that the deduction for the interiors is possible in polynomial time
       if $\ga$ is bounded by a constant, and so is for the exteriors,
       if $\ga$ or $|P(c)|$ is bounded by a constant, or if $|N(c)|$ is bounded
       by  a logarithm of the input size,  
where $P(c)$ corresponds to the set of positive literals in
       $c$. 

\item  As for Horn envelopes of the exteriors of Horn theories,  
we show that it is linearly solvable under model-based representation, 
while it is co-NP-complete under formula-based representation. 
The former contrasts to the negative result for the exteriors. 
We also present a polynomial algorithm for formula-based representation,
       if  $\ga$ is bounded by a constant or if $|N(c)|$ is bounded by a logarithm of the input
       size. 
\end{myitemize}

\newcommand{\lwb}[1]{\smash{\hbox{#1}}}
\begin{figure}
\label{fig-0}
\vspace*{-.3cm}
\begin{center}
\begin{small}
 \begin{tabular}{l|ccc}\hline\\[-.24cm]
 & Interiors & Exteriors & \ \ Envelopes of Exteriors\ \  \\[.1cm] \hline \\[-.24cm]
  Formula-Based \ & \lwb{P} & \lwb{co-NP-complete$^\star$}  & \ \
  \lwb{co-NP-complete$^\star$} \
  \  \\[.1cm]  
% \ &&&\\[.1cm]  
\hline \\[-.24cm]
  Model-Based  \ &  \ \   \lwb{NP-complete$^\dag$} \ \ &\ \  \lwb{co-NP-complete$^\ddag$} \
  \ & \lwb{P} \\[.1cm]  
%  \ &&&\\[.1cm]  
\hline%\bottomrule 
 \end{tabular}
\end{small}
\end{center}
\smallskip
$^\star$: It  is polynomially solvable, if $\ga=O(1)$ or 
 $|N(c)|=O(\log \parallel\!\Sigma\!\parallel)$.\\ \vspace*{-.3cm}

$^\dag$: It  is polynomially solvable, if $\ga=O(1)$.\\ \vspace*{-.3cm}

$^\ddag$: It is polynomially solvable, if $\ga=O(1)$, $|P(c)|=O(1)$, or
 $|N(c)|=O(\log  n|\charset(\Sigma)|)$. 
\caption{Complexity of the deduction for interiors and exteriors of
 Horn theories}
\end{figure}

The rest of the paper is organized as follows.
In the next section, we review the basic concepts and fix
notations.
Sections \ref{sec:formulabase} and \ref{sec-char}
investigate the deductive inference for the interiors and
exteriors of Horn theories. 
Section \ref{sec-env} considers the deductive inference for the
envelopes of the exteriors of  Horn theories.  

\section{Preliminaries}
\label{sec-def}
\subsection*{Horn Theories}
\label{sec-def-horn}
We assume a standard propositional language with atoms $\At =
\{ x_1, x_2,\ldots,x_n\}$, 
where each $x_i$ takes either value $1$ (true) or $0$ (false). 
A {\em literal}\/ is either an atom $x_i$ or its negation, which
 we denote by $\ol{x}_i$. 
The opposite of a literal $\ell$ is denoted by $\ol{\ell}$, 
and the opposite of a set of literals $L$ by $\ol{L}=\{ \ol{\ell} \mid \ell
\in L \}$.  
Furthermore, $\Lit = \At \cup \ol{\At}$ denotes the set of all literals.

A {\em clause} is a disjunction $c = \bigvee_{i \in P(c)} x_i \lor
\bigvee_{i \in N(c)}\ol{x}_i$ of literals, 
where $P(c)$ and $N(c)$ are the sets of indices whose 
corresponding variables  occur positively and
negatively in $c$ and $P(c) \cap N(c)=\emptyset$.  
Dually, a {\em term} is conjunction 
$t = \bigwedge_{i\in P(t)} x_i \land \bigwedge_{i \in N(t)}\ol{x}_i$ of
literals, where $P(t)$ and $N(t)$ are similarly defined. 
We also view clauses and terms as sets of literals.
% $P(c)\cup N(c)$ and $P(t)\cup
%N(t)$, respectively. 
A {\em conjunctive normal form} ({\em CNF}) is a conjunction of
clauses. 
A clause $c$ is {\em Horn}, if $|P(c)|\leq 1$. 
A {\em theory} $\Sigma$ is any set
of formulas; it is {\em Horn}, if it is a set of Horn clauses. As
usual, we identify $\Sigma$ with $\gvp= \bigwedge_{c\in \Sigma} c$,
and write $c \in \gvp$ etc.
It is known \cite{linear} that the deductive query for a Horn theory, i.e.,
deciding if $\Sigma \models c$ for a clause $c$ is possible in linear
time.

We recall that Horn theories have a well-known semantic
characterization. A {\em model} is a vector $v\!\in\!\{0,1\}^n$, whose
$i$-th component is denoted by $v_i$. 
%For $B\!\subseteq\!\{1,\ldots,n\}$, we let $x^B$ be the model $v$ such
%that $v_i=1$, if $i \in B$ and $v_i =0$, if $i \notin B$, for
%$i\!\in\!\{1,\ldots,n\}$. 
For a model $v$, let $\ON(v)=\{ i \mid v_i=1\}$ and $\OFF(v)=\{ i \mid
v_i=0\}$. 
The value of a formula $\gvp$ on a model
$v$, denoted $\gvp(v)$, is inductively defined as usual; satisfaction
of $\gvp$ in $v$, i.e., $\gvp(v)=1$, will be denoted by $v\models
\gvp$. 
The set of models of a formula $\gvp$ (resp.,
theory $\Sigma$), denoted by $\modd(\gvp)$ (resp., $\modd(\Sigma)$),
and logical consequence $\gvp \models \psi$ (resp., $\Sigma\models
\psi$) are defined as usual.
For two models $v$ and $w$, we denote by $v \leq w$ the usual componentwise
ordering, i.e., $v_i \leq w_i$ for all $i = 1,2,\ldots,n$, where
$0\leq 1$; $v < w$ means $v\neq w$ and $v \leq w$. 
%For any set of
%models $\cM$, we denote by $\max(\cM)$, 
%(resp., $\min(\cM)$) the set of
%all maximal (resp., minimal) models in $cM$.  
Denote by $v\,\AN\, w$
componentwise AND of models $v,w \in \{0,1\}^n$, and by $\ccap(\cM)$
the closure of $\cM \subseteq \{0,1\}^n$ under $\,\AN\,$.  Then, a theory
$\Sigma$ is Horn representable if and only if  $\modd(\Sigma) =
\ccap(\modd(\Sigma))$ (see \cite{dech-pear-92,khar-roth-96}) for
proofs).

\begin{exmp}
{\rm Consider $\cM_1 \!=\!\{(0101), (1001), (1000)\}$ and $\cM_2\!=\!\{
(0101)$, $(1001),$ $(1000),$ $(0001),$ $(0000)\}$. Then, for $v =
(0101)$, $w =(1000)$, we have $w,v \in \cM_1$, while $v\AN w= (0000)
\notin \cM_1$; hence $\cM_1$ is not the set of models of
a Horn theory. On the other hand, $\ccap(\cM_2) = \cM_2$, thus $\cM_2 =
\modd(\Sigma_2)$ for some Horn theory $\Sigma_2$.
}
\end{exmp}

As discussed by Kautz {\em et~al.} \cite{kaut-etal-93}, 
a Horn theory $\Sigma$ is
semantically represented by its characteristic models, where $v\in
\modd(\Sigma)$ is called {\em characteristic} (or {\em extreme}
\cite{dech-pear-92}), if $v \not\in \ccap(\modd(\Sigma) \setminus \{ v
\})$. 
The set of all such models, the {\em characteristic set of
$\Sigma$}, is denoted by $\charset(\Sigma)$.  
Note that
$\charset(\Sigma)$ is unique. E.g., $(0101) \in \charset(\Sigma_2)$,
while $(0000)\notin \charset(\Sigma_2)$; we have $\charset(\Sigma_2) =
\cM_1$.
It is known \cite{kaut-etal-93} that the deductive query for a Horn theory
$\Sigma$ from the characteristic set $\charset(\Sigma)$ can be done in
linear time, i.e., $O(n|\charset(\Sigma)|)$ time.

%%%%%%%%%%%%%%%            %%%%%%%%%%%%%%%%%%

\subsection*{Interior and Exterior of Theories}
For a model $v \in \{0,1\}^n$ and a nonnegative integer $\ga$, 
its {\em $\ga$-neighborhood}  is defined by \[
\cN_{\ga}(v) =\{ w \in \{0,1\}^n \mid \parallel w-v \parallel \leq \ga \}, 
\] 
where $\parallel v \parallel$ denotes $\sum_{i=1}^{n} |v_{i}|$. 
Note that $|\cN_{\ga}(v)|=\sum_{i=0}^{\ga} {n \choose i}=O(n^{\alpha+1})$. 
For a theory $\Sigma$  and a nonnegative integer $\ga$, 
the $\ga$-interior and $\ga$-exterior of $\Sigma$,  denoted by 
$\sigma_{-\ga}(\Sigma)$ and $\sigma_{\ga}(\Sigma)$ respectively, are theories defined by 
\begin{eqnarray}
\modd(\sigma_{-\ga}(\Sigma))&=& \{v \in \{0,1\}^n \mid  \cN_{\ga}(v)
 \subseteq  \modd(\Sigma) \}\\
\modd(\sigma_{\ga}(\Sigma))&=& \{v \in \{0,1\}^n \mid  \cN_{\ga}(v) \cap
\modd(\Sigma) \not=\emptyset\}.
 \end{eqnarray}
By definition,  $\sigma_{0}(\Sigma)=\sigma$, $\sigma_{\ga}(\Sigma) \models \sigma_{\gb}(\Sigma)$ for integers
$\ga$ and $\gb$ with $\ga < \gb$, and  
$\sigma_{\ga}(\Sigma_1) \models \sigma_{\ga}(\Sigma_2)$ holds for any integer
$\ga$,  if 
two theories $\Sigma_1$ and $\Sigma_2$ satisfy $\Sigma_1 \models \Sigma_2$.

\begin{exmp}\rm 
\label{ex-2}
Let us consider a Horn theory $\Sigma =\{\ol{x}_1 \vee x_3, \ol{x}_2 \vee x_3,
  \ol{x}_2\vee x_4\}$ of $4$ variables, where $\modd(\Sigma)$ is given by 
\[\modd(\Sigma)=\{(1111), (1011), (1010), (0111), (0011), (0010), (0001), (0000)
\} \] 
(See Figure \ref{fig-ex-aa}). 
Then we have  
$\sigma_{\ga}(\Sigma)=\{ \emptyset \}$ for $\ga\leq -2$, 
$\{\ol{x}_1,\ol{x}_2, x_3,  x_4 \}$ for $\ga=-1$, 
$\Sigma$ for $\ga=0$, $\{ \ol{x}_1 \vee
 \ol{x}_2 \vee x_3 \vee  x_4 \}$ for $\ga=1$, and 
$\emptyset$ for $\ga\geq 2$.
For example, $(0011)$ is the unique model of $\modd(\sigma_{-1}(\Sigma))$, 
since $\cN_{1}(0011) \subseteq \modd(\Sigma)$ and 
$\cN_{1}(v) \not\subseteq \modd(\Sigma)$ holds for all the other models $v$. 
For the $1$-exterior, we can see that 
all models $v$ with $(\ol{x}_1 \vee
 \ol{x}_2 \vee x_3 \vee  x_4)(v)=1$ satisfy $\cN_{1}(v) \cap
\modd(\Sigma) \neq \emptyset$, and no other such model exists. 
For example, $(0101)$ is a model of 
$\sigma_{1}(\Sigma)$, since $(0111) \in \cN_{1}(0101)\cap
\modd(\Sigma)$. On the other hand,  $(1100)$ is not a model of 
$\sigma_{1}(\Sigma)$, since $\cN_{1}(1100)\cap
\modd(\Sigma)=\emptyset$. 
 Notice that $\sigma_{-1}(\Sigma)$ is Horn, 
while  $\sigma_{1}(\Sigma)$ is not.
\end{exmp}
 \begin{figure}
\begin{center}
  \includegraphics[scale=.8]{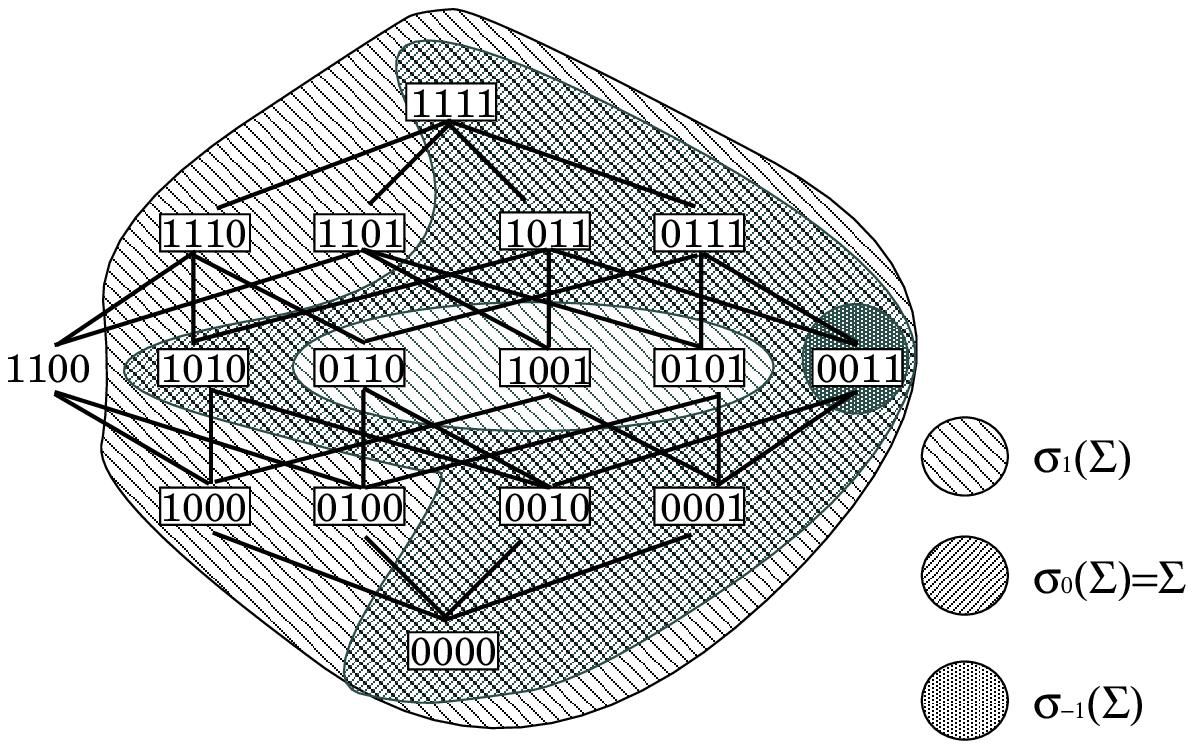}
\end{center}
\caption{A Horn theory and its interiors and exteriors}
\label{fig-ex-aa}
\end{figure}

Makino and Ibaraki \cite{interior} introduced the interiors and exteriors
to analyze  stability of Boolean functions, and studied their
basic properties and complexity issues on them (see also \cite{posinterior}).
 For example, it is known \cite{interior} 
that, for a theory $\Sigma$ and nonnegative integers $\ga$  and $\gb$, 
$\sigma_{-\ga}(\sigma_{-\gb}(\Sigma))= \sigma_{-\ga-\gb}(\Sigma)$, 
$\sigma_{\ga}(\sigma_{\gb}(\Sigma)) = \sigma_{\ga+\gb}(\Sigma)$, and 
\begin{eqnarray}
\sigma_{\ga}(\sigma_{-\gb}(\Sigma)) &\models& \sigma_{\ga-\gb}(\Sigma)
 \ \models  \sigma_{-\gb}(\sigma_{\ga}(\Sigma)). \label{eq--0}
\end{eqnarray}
For a nonnegative integer $\ga$ and two theories $\Sigma_1$ and
$\Sigma_2$, we have 
\begin{eqnarray}
  \sigma_{-\ga}(\Sigma_1 \cup \Sigma_2) &=& \sigma_{-\ga}(\Sigma_1) \cup
   \sigma_{-\ga}(\Sigma_2)\label{eq-1}\\
  \sigma_{\ga}(\Sigma_1 \cup \Sigma_2) &\models& \sigma_{\ga}(\Sigma_1) \cup
   \sigma_{\ga}(\Sigma_2), \label{eq-2}
\end{eqnarray}
where  $\sigma_{\ga}(\Sigma_1 \cup \Sigma_2) \not= \sigma_{\ga}(\Sigma_1) \cup
   \sigma_{\ga}(\Sigma_2)$ holds in general.

As demonstrated in Example \ref{ex-2}, it is not difficult to see that 
the interiors of any Horn theory are Horn, which is, for example, proved by \raf{eq-1} and Lemma  \ref{prop3}, while the exteriors might be not Horn.

%%%%%%%%%%%%%%%%%%%%%%%%%%%%%%%%%%

\nop{%%%%%%%%%%%%%%%%%%%5

It should be noticed that exterior functions of Horn functions are also
Horn by Property \ref{prop3}. However, this is not always true for the interior
functions.

\begin{property}\label{prop4}\cite{interior}
 Let $f$ be a function and $\theta$ be an integer. Then 
 $\sigma_{\theta}(\bar{f})=\overline{\sigma_{-\theta}(f)}$ holds. 
\end{property}
}%%%%%%%%%%%%%%%%%%%%%%%%%%%%

\section{Deductive  Inference from Horn Theories}
\label{sec:formulabase}
In this section, we investigate the deductive inference for the interiors and
exteriors of a given Horn theory. 

\subsection{Interiors}
Let us first consider the deduction for the $\ga$-interiors of a  Horn
theory: 
 Given a Horn theory $\Sigma$,
a clause $c$, and a positive integer $\ga$, decide if $\sigma_{-\ga}(\Sigma)
\models c$ holds. 
We show that the problem is solvable in linear time after showing a
series of lemmas. 

The following lemma is a basic property of the interiors.

\begin{lem}\label{prop3}
Let $c$ be a clause. Then for a nonnegative integer $\ga$, 
we have $\sigma_{-\ga}(c)=\OR_{{S \subseteq
  c:}\atop{|S|=\ga+1}}\bigl(\AN_{\ell \in S}\ell\bigr) =\AN_{{S \subseteq
  c:}\atop{|S|=|c|-\ga}}\bigl(\OR_{\ell \in S}\ell\bigr)$.  
\end{lem}

For example, let us consider $c=x_1 \vee x_2 \vee \ol{x}_3 \vee
\ol{x}_4$, $\ga=2$. Then we have $\sigma_{-\ga}(c)=x_1 x_2 \ol{x}_3 \vee
x_1 x_2 \ol{x}_4 \vee x_1 \ol{x}_3\ \ol{x}_4 \vee x_2 \ol{x}_3\
\ol{x}_4=(x_1 \vee x_2)(x_1 \vee \ol{x}_3)(x_1 \vee \ol{x}_4)(x_2 \vee \ol{x}_3)(x_2 \vee \ol{x}_4)(\ol{x}_3 \vee \ol{x}_4)$.

This lemma, together with  \raf{eq-1}, implies that 
for a CNF $\gvp$ and a nonnegative integer $\ga$, we have 
\begin{eqnarray*}
 \sigma_{-\ga}(\gvp)&=&\AN_{ c \in \gvp}\Bigl(\OR_{{S \subseteq
  c:}\atop{|S|=\ga+1}}\bigl(\AN_{\ell \in S}\ell\bigr)\Bigr) \ =\ 
\AN_{ c \in \gvp}\Bigl(\AN_{{S \subseteq
  c:}\atop{|S|=|c|-\ga}}\bigl(\OR_{\ell \in S}\ell\bigr)\Bigr), 
\end{eqnarray*}
where we regard $c$ as a set of literals.

\begin{lem}\label{lemma-2}
Let $\Sigma$ be a Horn theory, and let $c$ be a clause. 
For a nonnegative integer $\ga$, 
if there
exists a clause $d \in \Sigma$ such that $|N(d)\setminus N(c)| \le
\ga-1$ or $(|N(d)\setminus N(c)|=\ga$ and $P(d) \subseteq P(c))$,
then  we have  $\sigma_{-\ga}(\Sigma) \models c$. 
\end{lem}

\begin{proof}
If $\Sigma$ has a  clause $d$ such that $|N(d)\setminus N(c)| \le
\ga-1$, then  $|(N(d)\setminus N(c)) \cup P(d)| \leq \ga$ holds. 
Thus by Lemma \ref{prop3}, we have 
$\sigma_{-\ga}(d) \models \OR_{i \in  N(c) \cap N(d)}\ol{x}_i \models c$.
Therefore,  by \raf{eq-1}, $\sigma_{-\ga}(\Sigma) \models c$ holds. 

On the other hand, 
if $\Sigma$ has a  clause $d$ such that $|N(d)\setminus N(c)|= \ga$ and  $P(d) \subseteq P(c)$, then  
 by Lemma \ref{prop3},  we have 
$\sigma_{-\ga}(d) \models \OR_{i \in P(c)} x_i \vee \OR_{i \in  N(c) \cap N(d)}\ol{x}_i \models c$.
Therefore,  by \raf{eq-1}, $\sigma_{-\ga}(\Sigma) \models c$ holds. 
\end{proof}

\begin{lem}\label{lemma-1}
Let $\Sigma$ be a Horn theory, and let $c$ be a clause. 
For a nonnegative integer $\ga$, 
if {\rm (i)} $|N(d)\setminus N(c)| \ge \ga$ holds 
for all $d \in \Sigma$ and {\rm (ii)} $\emptyset \not= P(d) \subseteq N(c)$ holds for
 all  $d \in \Sigma$  with  $|N(d)\setminus N(c)| = \ga$, 
then we have  $\sigma_{-\ga}(\Sigma)\not\models c$. 
\end{lem}

\begin{proof}
Let $v$ be the unique minimal model that does not satisfy $c$, 
i.e., $v_i=1$ if $\ol{x}_i \in c$ and $0$, otherwise. 
We show that $v \models \sigma_{-\ga}(\Sigma)$, 
which  implies $\sigma_{-\ga}(\Sigma)\not\models c$. 

Let $d$ be a clause in $\Sigma$ with $|N(d)\setminus N(c)| \ge
\ga+1$, 
and let $t$ be a term obtained by conjuncting arbitrary $\ga+1$ literals in $N(d)\setminus N(c)$. 
Then we have $t(v)=1$ and $t \models \sigma_{-\ga}(d)$ by Lemma
\ref{prop3}.  
On the other hand, for a clause $d$ in $\Sigma$ with  $|N(d)\setminus N(c)| = \ga$, 
let $t$ be a term obtained by conjuncting all literals in
$(N(d)\setminus N(c)) \cup P(d)$. 
Then we have $|t|=\ga+1$ and $t \models \sigma_{-\ga}(d)$ by Lemma
\ref{prop3}.  Moreover, it holds that $t(v)=1$ by $P(d) \subseteq N(c)$.
Therefore, by \raf{eq-1}, we have $v \models \sigma_{-\ga}(\Sigma)$. 
\end{proof}

By Lemmas \ref{lemma-2} and \ref{lemma-1}, we can easily answer 
the deductive queries, if $\Sigma$ satisfies certain conditions mentioned in
them. 
In the remaining case, 
we have the following lemma.

\begin{lem}\label{lemma-3}
For a Horn theory $\Sigma$ that satisfies none of the conditions
in Lemmas {\rm \ref{lemma-2}} and {\rm \ref{lemma-1}}, 
let  $d$ be a clause  in $\Sigma$ such that $|N(d)\setminus N(c)| =
\ga$, and $P(d)=P(d) \setminus (P(c) \cup N(c))=\{j \}$. 
Then $\sigma_{-\ga}(\Sigma) \models c \vee x_j$ holds.
\end{lem}

\begin{proof}
By Lemma \ref{prop3}, we have $\sigma_{-\ga}(d) \models \OR_{i \in  N(c) \cap
 N(d)} \ol{x}_i  \vee x_j \models c \vee x_j$.
This implies  $\sigma_{-\ga}(\Sigma) \models c \vee x_j$ by \raf{eq-1}. 
 \end{proof}

{From} this lemma, we have only to check a deductive query $\sigma_{-\ga}(\Sigma)
\models c \vee \ol{x}_j$, instead of $\sigma_{-\ga}(\Sigma)
\models c$. 
Since  $|c| < |c \vee \ol{x}_j| \leq n$, 
we can answer the deduction by checking the conditions in Lemmas {\rm
\ref{lemma-2}} and {\rm \ref{lemma-1}} at most $n$ times.

\nop{%%%%%%%%%%%%%%%%%%%%%%%%%%%
i.e., the case that satisfies
all three conditions below. 
\begin{description}
\item[\hspace*{.3cm}{\rm (I)}]
$|N(c')\setminus N(c)| \ge \ga$ holds  for all $c' \in \Sigma$, 
\item[\hspace*{.3cm}{\rm (I$\!$I)}] 
$P(c') \cap P(c)=\emptyset$ holds for all $c' \in \Sigma$ with $|N(c')\setminus N(c)|=\ga$, and 

\item[\hspace*{.3cm}{\rm (I$\!$I$\!$I)}] \parbox[t]{11cm}
{there exists a clause $c'$ in $\Sigma$ such that  $|N(c')\setminus N(c)| = \ga$ and $P(c') \cap N(c) =\emptyset$.}
\end{description}
%%%%%%%%%%%%%%%%%%%%%%%%%%%%%%%%%%%%%%%%%%%%%%%%%%
Let 
\begin{eqnarray}
\Sigma^*&=&\{ c' \in \Sigma \mid |N(c')\setminus N(c)| = \ga, P(c') \cap N(c) =\emptyset
\},  
\end{eqnarray}
and let $P=\bigcup_{c' \in \Sigma^*}P(c')$. 
}%%%%%%%%%%%%%%%%%%%%%%%%%%%%%

\begin{algorithm}[htbp]
 \caption{{Deduction-Interior-from-Horn-Theory}}
\begin{description}
\setlength{\itemsep}{0.1cm}
 \item[Input:] A Horn theory
	    $\Sigma$, a clause $c$ and a nonnegative integer $\ga$. 

 \item[Output:] Yes, if $\sigma_{-\ga}(\Sigma) \models c$; Otherwise,
		      No.
\smallskip

 \item[Step\,0.] Let $N:=N(c)$ and $P:=P(c)$.

 \item[Step\,1.] /* Check the condition in Lemma \ref{lemma-2}. */ \\
{\bf If} there
exists a clause $d \in \Sigma$ such that $|N(d)\setminus N| \le
\ga-1$ or $(|N(d)\setminus N|=\ga$ and $P(d) \subseteq P)$, 
{\bf then} output Yes and halt. 

 \item[Step\,2.] /* Check the condition in Lemma \ref{lemma-1}. */ \\
{\bf If} $P(d) \subseteq N$ holds for all  $d \in \Sigma$  with  $|N(d)\setminus N| = \ga$, 
{\bf then} output No and halt.

 \item[Step\,3.] /* Update $N$ by Lemma \ref{lemma-3}. */ \\
For a clause $d$ in $\Sigma$ such that $|N(d)\setminus N| =
\ga$ and $P(d)=P(d) \setminus (P\cup N)=\{j \}$,  
update $N:=N \cup \{j\}$ and return to Step 1.  \qed
\end{description}
\end{algorithm}

We can see that a straightforward implementation of the algorithm
requires $O(n(\parallel \!\!\Sigma \!\!\parallel+|c|))$ time, 
where $\parallel\!\! \Sigma \!\!\parallel$ denotes the length of $\Sigma$, i.e., 
$\parallel \!\!\Sigma \!\!\parallel=\sum_{d \in \Sigma}|d|$.  
However, it is not difficult to see that we have a linear time algorithm for
the problem, if $N(d)\setminus
N$ for $d \in \Sigma$ is maintained  by using the proper data structure.

\begin{thm}
\label{th-interior-horn-theory}
Given  a Horn theory $\Sigma$, a clause $c$ and a nonnegative integer $\ga$, 
a deductive query $\sigma_{-\ga}(\Sigma) \models c$ can be answered in 
linear time, i.e., $O(\parallel \!\!\Sigma \!\!\parallel+|c|)$ time. \qed
\end{thm}

\subsection{Exteriors}
Let us next consider the deduction for the $\ga$-exteriors of a  Horn
theory. 
In contrast to the interior case, we have the following negative
result. 

\begin{thm}
\label{theo:int:f:g}
Given  a Horn theory $\Sigma$, a clause $c$ and a positive integer $\ga$, 
it is {\rm co}-{\rm NP}-complete to decide whether a deductive query 
 $\sigma_{\ga}(\Sigma) \models c$ holds, even if $P(c)=\emptyset$.  
\end{thm}

\begin{proof}%%%%%%%% Proof of Theorem 2  %%%%%%%%%%%%
By definition, $\sigma={\ga}(\Sigma) \not\models c$ if and only if 
there exists a model $v$ of $\Sigma$ such that some model in $\cN_{\ga}(v)$
does not satisfy $c$. The latter is equivalent to the condition that 
there exists a model $v$ of $\Sigma$ such that 
$|\ON(v) \cap P(c)|+|\OFF(v) \cap N(c)| \le \ga$, which can be checked in polynomial time. Thus the problem is in co-NP.

We then show the hardness by reducing a well-known NP-complete problem
 {\sc Independent Set} to the complement of our  problem. {\sc
 Independent Set} is the problem of deciding if a given graph $G=(V,E)$
 has an independent set $W \subseteq V $ such that $|W| \geq k$ for
 a given integer $k$. Here we call a subset $W \subseteq V$ is an {\em
 independent set} of $G$ if $|W \cap e| \leq 1$ for all edges $e \in E$.
For a problem instance $G=(V=\{1,2,\ldots,n\},E)$ and $k$ of  {\sc Independent Set},
let us define a Horn theory $\Sigma_G$ over $\At=\{x_1,x_2,\ldots,x_n\}$
by
\[
 \Sigma_{G} = \{ (\bar{x}_i \vee \bar{x}_j) \mid \{i,j\}\in E\}.
\]
Let $c=\bigvee_{i=1}^n \ol{x}_i$ and $\ga=n-k$.
Note that $(11\cdots 1)$ is the unique model that does not satisfy 
$c$. 
Thus $\sigma_{\ga}(\Sigma)\not\models c$ if and only if 
$\sigma_{\ga}(\Sigma)(11\cdots 1)=1$.  
Since $W$ is an independent set of $G$ if and only if 
$\Sigma_{G}$ contains a model $w$ defined by $\ON(w)=W$, 
$\sigma_{\ga}(\Sigma_{G})(11\cdots 1)=1$ is equivalent to the condition that 
$G$ has an independent set of size at least $k\,(=n-\ga)$. 
This completes the proof. 
\end{proof}

We remark that this result can also be derived from the ones in
\cite{interior}. 

However, by using the next lemma, 
a deductive query can be answered in polynomial time, if $\ga$ or $N(c)$ is
small.

\begin{lem}\label{intex}
Let $\Sigma_1$ and 
$\Sigma_2$ be theories. For a nonnegative integer $\ga$, 
Then $\sigma_{\ga} (\Sigma_1) \models \Sigma_2$  
if and only if 
$\Sigma_1 \models 
\sigma_{-\ga}(\Sigma_2)$.  
\end{lem}

\begin{proof}
For the if part, 
if $\Sigma_1 \models \sigma_{-\ga}(\Sigma_2)$, then 
we have $\sigma_{\ga}(\Sigma_1) $ $\models
\sigma_{\ga}(\sigma_{-\ga}(\Sigma_2)) \models \Sigma_2$ by \raf{eq--0}. 
On the other hand, if $\sigma_{\ga} (\Sigma_1) \models \Sigma_2$, 
then we have $ \Sigma_1 \models \sigma_{-\ga}(\sigma_{\ga} (\Sigma_1)) \models 
\sigma_{-\ga}(\Sigma_2)$ by \raf{eq--0}. 
\end{proof}

%Let $\Sigma$ be a Horn theory, let $c$ be a clause, and let $\ga$ be a
%nonnegative integer. 
{From} Lemma \ref{intex}, the deductive query for the $\ga$-interior of
a theory $\Sigma$,
i.e.,  $\sigma_{\ga} (\Sigma) \models c$ for a given clause $c$ is equivalent to the condition
that $\Sigma \models 
\sigma_{-\ga}(c)$.  
Since we have $\sigma_{-\ga}(c)=\AN_{{S \subseteq
  c:}\atop{|S|=|c|-\ga}}\bigl(\OR_{\ell \in S}\ell\bigr)$ by Lemma \ref{prop3}, 
the deductive query for the $\ga$-interior can be done by checking
${{|c|}\choose{\ga}}$ deductions for $\Sigma$. 
More precisely, 
we have the following lemma.

\begin{lem}
\label{lemma-adfe}
Let $\Sigma$ be a Horn theory, let $c$ be a clause, 
and $\ga$ be  a nonnegative integer. 
Then $\sigma_\ga(\Sigma) \models c$ holds if and only if, for each 
subset $S$ of $N(c)$ such that $|S| \geq |N(c)|-\ga$, 
at least  $(\ga-|N(c)|+|S|+1)$ 
$j$'s  in $P(c)$ satisfy $\Sigma \models \OR_{i \in S}\ol{x}_i \vee
 x_j$.  
\end{lem}

\begin{proof}
{From} Lemmas \ref{prop3} and \ref{intex}, 
$\sigma_{\ga}(\Sigma) \models c$ if and only if  $\Sigma \models \AN_{{S \subseteq
  c:}\atop{|S|=|c|-\ga}}\bigl(\OR_{\ell \in S}\ell\bigr)$. 
It is known that for a Horn theory $\Sigma$ and clause $d$, $\Sigma
 \models d$ if and only if $\Sigma \models \OR_{i \in N(d)} \ol{x}_i \vee x_j$
 holds for some $j \in P(d)$ (i.e., All the prime implicates of Horn
 theory are Horn). 
This proves the lemma. 
\end{proof}

This lemma implies that the deductive query can be answered by checking 
the number of $j$'s  in $P(c)$ that satisfy 
$\Sigma \models \OR_{i \in S}\ol{x}_i \vee
 x_j$ for each $S$. 
Since we can check this condition in linear time and 
there are  $\sum_{p=0}^\ga{{|N(c)|}\choose{p}}$ such $S$'s, we have the
following result, which complements Theorem \ref{theo:int:f:g} that the
problem is intractable, even if $P(c)=\emptyset$.  

\begin{thm}
\label{th-adfe}
Let $\Sigma$ be a Horn theory, let $c$ be a clause, and let $\ga$ be a
nonnegative integer. 
Then a deductive query $\sigma_{\ga}(\Sigma) \models c$ can
 be  answered in $O\Bigl( \sum_{p=0}^\ga{{|N(c)|}\choose{p}}
 \parallel\!\Sigma\!\parallel+ |P(c)|\Bigr)$
 time. 
In particular, it  is polynomially solvable, if $\ga=O(1)$  or
 $|N(c)|=O(\log \parallel\!\Sigma\!\parallel)$.
\end{thm}

\section{Deductive  Inference from Characteristic Sets}
\label{sec-char}
In this section, we consider the case when Horn knowledge bases can be
represented by characteristic sets. 
Different from formula-based representation, the deductions for interiors
and exteriors are both intractable, unless P=NP. 

\subsection{Interiors}
We first present an algorithm to solve the deduction problem for the
interiors of Horn theories. 
The algorithm requires exponential time in general, but it is
polynomial when $\ga$ is small. 

Let $\Sigma$ be a Horn theory given by its characteristic set
$\charset(\Sigma)$, and let $c$ be a clause. 
Then for a nonnegative integer $\ga$, we have 
\begin{equation}
\label{eq-model-1}
\sigma_{-\ga}(\Sigma) \models c \mbox{ if and only if }
\sigma_{-\ga}(\Sigma) \wedge \ol{c}\equiv 0. 
\end{equation}
Let $v^*$ be a unique minimal model such that $c(v^*)=0$ (i.e., $\ol{c}(v^*)=1$). 
By the definition of interiors, $v^*$ is a model of $\sigma_{-\ga}(\Sigma)$ if and only if
all $v$'s in $\cN_{\ga}(v^*)$ are models of $\Sigma$. 
Therefore, for each model $v$ in $\cN_{\ga}(v^*)$, we check if $v \in
\modd(\Sigma)$, which is equivalent to  
\begin{equation}\label{eq:extmodel}
\bigwedge_{w \in \charset(\Sigma) \atop w \ge v}\!\!w \ =  \ v. 
\end{equation}
If \raf{eq:extmodel} holds for all models $v$ in $\cN_{\ga}(v^*)$, 
then we can  immediately conclude by \raf{eq-model-1} that
$\sigma_{-\ga}(\Sigma) \not\models c$. 
On the other hand, if there exists a model $v$ in $\cN_{\ga}(v^*)$ such
that \raf{eq:extmodel} does not hold, 
let $J=ON(\bigwedge_{w\in
\charset(\Sigma) \atop w \ge v} w)
\setminus ON(v)$. 
By definition, we have $J\not=\emptyset$, and 
we can see that 
\begin{eqnarray}
\sigma_{-\ga}(\Sigma)& \models& \OR_{i \in ON(v)}\ol{x}_i
\vee x_j \ \   \mbox{ for all } j \in J.  
\label{eq-aaae}
\end{eqnarray}
If $J \cap N(c)\not=\emptyset$, then by Lemma \ref{prop3} and \raf{eq-aaae}, 
we have $\sigma_{-\ga}(\Sigma) \models \OR_{i \in ON(v) \cap N(c)}\ol{x}_i$, 
since $|ON(v) \setminus N(c)| \leq \ga-1$. This implies 
$\sigma_{-\ga}(\Sigma) \models c$. 
On the other hand, if $J \cap N(c)=\emptyset$, then by Lemma \ref{prop3} and \raf{eq-aaae}, 
we have $\sigma_{-\ga}(\Sigma) \models \OR_{i \in N(c)}\ol{x}_i
\vee x_j$ for all $j\in J$.   
Thus,  if $J$ contains an index in $P(c)$, then we can conclude that
$\sigma_{-\ga}(\Sigma) \models c$; Otherwise, we check the condition
$\sigma_{-\ga}(\Sigma) \models c \vee \OR_{j \in J}\ol{x}_j$, instead of
$\sigma_{-\ga}(\Sigma) \models c$. 
Since  a new clause $d=c \vee \OR_{j \in J}\ol{x}_j$ is longer than
$c$, after at most $n$ iterations, we can answer the deductive query. 
Formally, our algorithm can be described as Algorithm \ref{alg:DIC}. 

\begin{algorithm}
 \caption{{Deduction-Interior-from-Charset}}
\label{alg:DIC}
\begin{description}
\setlength{\itemsep}{0.1cm}
 \item[Input:] The characteristic set $\charset(\Sigma)$ of a Horn theory
	    $\Sigma$, a clause $c$ and a nonnegative integer $\ga$. 

 \item[Output:] Yes, if $\sigma_{-\ga}(\Sigma) \models c$; Otherwise,
		      No.
\smallskip

 \item[Step\,0.] Let $N:=N(c)$,  $d:=c$ and $q:=1$.   
 \item[Step\,1.] Let $v^*$ be the unique minimal
		      model such that $d(v^*)=0$.

 \item[Step\,2.] {\bf For} each $v$ in $\cN_{\ga}(v^*)$ {\bf do}\\[.1cm]  
\hspace*{1.4cm}
{\bf If} \raf{eq:extmodel} does not hold, \\[.1cm]  
\hspace*{2cm}{\bf then} let $v^{(q)}:=v$,  
$J:=ON(\bigwedge_{w\in
\charset(\Sigma) \atop w \ge v} w)
\setminus ON(v)$ and \\
\hspace*{2.4cm} $q:=q+1$  \\[.1cm]  
\hspace*{2cm}{\bf If} $J \cap (N \cup P(c))\not=\emptyset$, {\bf then} output yes and halt. \\[.1cm]  
\hspace*{2cm}Let $N:= N \cup J$ and $d:=\OR_{i \in N}\ol{x}_i \vee
		      \OR_{i \in P(c)} x_i$. \\[.1cm]  
\hspace*{2cm}Go to Step 1. \\[.1cm]  
\hspace*{1cm}{\bf end}$\{$for$\}$            
\item[Step\,3.] Output No and halt.  \qed
\end{description}
\end{algorithm}

\begin{thm}
\label{th-interior-charset}
Given  the characteristic model $\charset(\Sigma)$ of a Horn theory
 $\Sigma$, 
a clause $c$ and a nonnegative integer $\ga$, 
a deductive query $\sigma_{-\ga}(\Sigma) \models c$ can be answered in 
 $O( n^{\ga+2} |\charset(\Sigma)|)$ time. In particular, it is
 polynomially solvable, if $\ga=O(1)$. 
\end{thm}

\begin{proof}
Since we can see algorithm {\sc Deduction-Interior-from-Charset} correctly answers
a deductive query from the discussion before the description, we only
estimate  the running time of the algorithm.  

Steps 0, 1 and 3  require $O(n)$ time. 
Step 2 requires $O(n^{\ga+1}|\charset(\Sigma)|)$ time, since
 \raf{eq:extmodel} can be checked in 
$O(n|\charset(\Sigma)|)$ time. 
Since we have at most $n$ iterations between Steps 1 and 2, 
the algorithm requires   $O( n^{\ga+2} |\charset(\Sigma)|)$ time. 
\end{proof}

However, in general, the problem is intractable, which contrasts with
the formula-model representation.

\begin{thm}
\label{theo-interior-char}
Given  the characteristic set $\charset(\Sigma)$ of a Horn theory
 $\Sigma$ and a positive integer $\ga$, 
it is {\rm co}-{\rm NP}-complete to decide whether
 $\sigma_{-\ga}(\Sigma)$ is consistent, i.e., 
 $\modd(\sigma_{-\ga}(\Sigma))\not=\emptyset$. 
\end{thm}

\begin{proof}
Let us first show the co-NP-completeness of the problem.
Apply Algorithm {\sc Deduction-Interior-from-Charset} to the instance 
($\charset(\Sigma)$, $c=\emptyset$, $\ga$). 
If $\sigma_{-\ga}(\Sigma)$ is not consistent, then the algorithm
 constructs vectors $v^{(1)}, \dots , v^{(k)}$, $k\le n$, in Step 2. 
We can see that these vectors form a witness to the inconsistency of
 $\sigma_{-\ga}(\Sigma)$.  In fact, if we are given these vectors, the
 inconsistency can be checked in polynomial time. This implies that the
 problem belongs to co-NP. 

We show the co-NP-hardness by reducing {\sc Independent
 Set} to our problem. Given a 
 problem instance 
 $G=(V=\{1,2,\ldots,n\},E)$ and $k$ of {\sc Independent
 Set}, let us define a Horn theory
 $\Sigma_{G}$ over $\At = \{x_1,x_2,$ $\ldots,x_n\}$ by
\[
 \charset(\Sigma_G) = \{ v^{(i,j)}, v^{(i,j,l)} \mid \{i,j\} \in E, l
 \in V \setminus \{i,j\} \},
\]
where $v^{(i,j)}$  and $v^{(i,j,l)}$ are respectively 
the vectors defined by $\OFF(v^{(i,j)})=\{i,j\}$ and $\OFF(v^{(i,j,l)})=\{i,j,l\}$.
Let $\ga = n-k$. 
Note that $\Sigma_G$ is a negative theory, and hence 
 $\sigma_{-\ga}(\Sigma_G)$ is consistent if and only if $(00\cdots0)$ is
 a model of  $\sigma_{-\ga}(\Sigma_G)$. 
Moreover, the latter condition is equivalent to the one that 
$G$ has no independent set of size at least $k\,(=n-\ga)$. 
 This completes the proof. 
\end{proof}

This result immediately implies the following corollary. 

\begin{cor}
\label{cor-interior-char}
Given  the characteristic set $\charset(\Sigma)$ of a Horn theory
 $\Sigma$, 
a clause $c$ and a positive integer $\ga$, 
it is {\rm NP}-complete to decide whether a deductive query 
 $\sigma_{-\ga}(\Sigma) \models c$ holds, even if $c=\emptyset$.  \qed
\end{cor}
\nop{
\begin{proof}
Similarly to the proof of Theorem \ref{theo-interior-char}, we can see that 
the problem belongs to NP. The hardness follows from Theorem \ref{theo-interior-char}. 
\end{proof}
}
Note that, different from the other hardness results, 
the hardness is not sensitive to the size of $c$.

\subsection{Exteriors}
Let us consider the exteriors. 
Similarly to the formula-based representation, we have the following
negative result. 

\begin{thm} 
\label{theo-exterior-char}
Given  the characteristic set $\charset(\Sigma)$ of a Horn theory
 $\Sigma$, 
a clause $c$ and a positive integer $\ga$, 
it is {\rm co}-{\rm NP}-complete to decide if a deductive query 
 $\sigma_{\ga}(\Sigma) \models c$ holds. 
\end{thm}

\begin{proof} 
{From} Lemmas \ref{prop3} and  \ref{intex}, 
$\sigma_\ga(\Sigma) \not\models c$ if and only if 
there exists a subclause $d$ of $c$ such that $|d|=|c|-\ga$ and 
$\Sigma \not\models d$. 
This $d$ is a witness that the problem belongs to co-NP.

We then show the hardness by a reduction from {\sc Vertex Cover} which
 is known to be NP-hard. 
{\sc Vertex Cover} is the problem to decide if a given graph $G=(V,E)$
 has a vertex cover $U$ such that $|U| \leq k$ for a given integer $k\,( < n)$.  
Here $U \subseteq V$ is called {\em vertex cover} if $U \cap e\not=
 \emptyset$ holds for all $e \in E$.  
For this problem instance,  we construct our problem instance. 
For each $e \in E$, let $W_e=\{e_1,e_2,\ldots, e_{|V|}\}$, and let 
$W=\bigcup_{e\in E} W_e$. 
Let $m^{(v)}$, $v \in V$, be a model over $V \cup W$ such that 
\[
 \ON(m^{(v)})=(V \setminus \{v\}) \cup \bigcup_{v \not\in
 e}W_e,  
\]
and let $\charset(\Sigma)$ be the characteristic set for some Horn theory
 $\Sigma$ defined by 
$ \charset(\Sigma)  =  \{ m^{(v)} \mid v \in V\}$. 
We define $c$ and $\ga$ by 
\[
 c  =  \OR_{i\in V} \ol{x}_i \vee  \OR_{i\in W} x_i \ \mbox{ and }\  
\ga=k, 
\]
respectively. For this instance, we show that $\sigma_\ga(\Sigma)
 \not\models c$  if and only if the corresponding $G$ has a vertex cover
 $U$ of size at most $k\,(=\ga)$. 

For the if part, let $U$ be such a vertex cover of $G$. 
For this $U$, we consider model $m^{(U)} \df \bigwedge_{v\in U}
 m^{(v)}$, 
which is a model of $\Sigma$ by the intersection property of a Horn theory.   
Note that $m^{(U)}$ does not satisfy a clause $d= 
\OR_{i\in V\setminus U} \ol{x}_i \vee \OR_{i\in W} x_i$. 
Since $d$ is a subclause of $c$ of length at least $|c|-\ga$, 
$m^{(U)}$ is not a model of $\sigma_{-\ga}(c)$  by Lemma \ref{prop3}. 
This completes the if part by Lemma \ref{intex}.

For the only-if part, let us assume that $\sigma_\ga(\Sigma) \not\models
 c$.
Then by Lemmas \ref{prop3} and  \ref{intex}, 
there exists a subclause $d$ of $c$ such that $|d|=|c|-\ga$ and 
$\Sigma \not\models d$. 
This implies that $\Sigma \wedge \ol{d}$ contains a model $m$. 
By $\ga< n$, for each $e \in E$, there exist an index $j$ in $W_e$ such that 
$m_j=0$.
Since any model $m'$ in $\Sigma$ satisfy either $m'_i=0$ or $m'_i=1$ for
 all $i \in W_e$, 
we have $m_i=0$ for all $i \in W$. This means that $V\setminus \ON(m)$
 is a vertex cover of $G$, and since $|V\setminus \ON(m)| \leq k$, 
we have the only-if part. 
\end{proof}

By using Lemma \ref{lemma-adfe}, we can see that 
the problem can be solved in polynomial time, if $\ga$ or $|N(c)|$ is small. 
Namely, for each 
subset $S$ of $N(c)$ such that $|S| \geq |N(c)|-\ga$, 
let $v^S$ denotes the model such that $\ON(v^S)=S$. 
Then $w^S=\bigwedge_{w\in \charset(\Sigma): \atop w\ge v^S}w$ 
is the unique minimal model of $\Sigma$ such that $\ON(w^S) \supseteq
S$, 
and hence it follows from Lemma \ref{lemma-adfe} that 
it is enough to  check if 
$|\ON(w^s) \cap P(c)| \geq \ga-|N(c)|+|S|+1$.  
Clearly, this can be done in 
in $O\left( \sum_{p=0}^\ga{{|N(c)|}\choose{p}}
 n|\charset(\Sigma)|\right)$ time.

Moreover, if $|P(c)|$ is small, then the problem also become tractable, 
which contrasts with Theorem \ref{theo:int:f:g}.

\begin{lem}
\label{lemma-adfe2}
Let $\Sigma$ be a Horn theory, let $c$ be a clause, 
and $\ga$ be  a nonnegative integer. Then
$\sigma_\ga(\Sigma) \models c$ holds if and only if  each $S
 \subseteq P(c)$ such that $|S| \geq |P(c)|-\ga$ satisfies 
 \begin{equation}\label{eq:adfe2}
  |\OFF(w) \cap N(c)| \geq \ga- |P(c)|+|S|+1
 \end{equation}
 for all models $w$ of $\Sigma$ such that 
$\OFF(w)\cap P(c)=S$. 
\end{lem}

\nop{%%%%
%\begin{proof}
%Similarly to the proof of Lemma \ref{lemma-adfe}, we have
% $\sigma_{\ga}(\Sigma) \models c$ if and only if  $\Sigma \models
% \AN_{{S \subseteq c:}\atop{|S|=|c|-\ga}}\bigl(\OR_{\ell \in
% S}\ell\bigr)$. It is easily shown to be equivalent to that
%for each $S\subseteq P(c)$,
%\begin{equation}\label{eq:adfe3}
% \left(\bigvee_{c'\in C_S}\ol{c'}\right)\left(\bigwedge_{w\in
%  W}w\right)=0,
%\end{equation}
%holds for every
%$W\subseteq \charset(\Sigma)$,
%where $C_S = \{c' \mid P(c')=S, N(c')\subseteq N(c) \mbox{ and
% }|N(c)\setminus N(c')|+|P(c)\setminus S|\le \ga\}$.
%We can restrict the candidates of $W$ such that
%for any $j\in S$, $W$ contains some $w\in W$ with $j\in OFF(w)$.
%This is because models generated by other $W$'s clearly satisfy
%equation (\ref{eq:adfe3}).
%Since (\ref{eq:adfe3}) is equivalent to
%$|N(c)\cap ON(\bigwedge_{w\in W}w )|< |N(c)|-\ga+|P(c)\setminus S|$, the
% lemma is proved. \qed
%\end{proof}
}

Note that \raf{eq:adfe2} is monotone in the sense that, 
  if 
a model $w$ satisfies \raf{eq:adfe2}, 
then all models $v$ with $v < w$ also satisfy
it. 
Thus it is sufficient to check if  \raf{eq:adfe2} holds for all {\em maximal} models $w$ of
$\Sigma$ such that $\OFF(w)\cap P(c)=S$. 
Since such maximal models $w$ can be obtained from $w^{(i)}$  $(i \in S)$
 with $ i \in \OFF(w^{(i)})\cap P(c) \subseteq S$ by their
  intersection $w=\AN_{i \in S} w^{(i)}$, 
we can answer the deduction problem in 
$O\left(n \sum_{p=|P(c)|-\ga}^{|P(c)|}
 {|P(c)| \choose p}|\charset(\Sigma)|^p\right)$ time.

\begin{thm}
\label{th-adfea}
Given the characteristic set $\charset(\Sigma)$ of a Horn theory,   a
 clause $c$, and a
nonnegative integer $\ga$,  
a deductive query $\sigma_{\ga}(\Sigma) \models c$ can
 be  answered in $O\Bigl( n \min \{\sum_{p=0}^\ga{{|N(c)|}\choose{p}}
 |\charset(\Sigma)|,  \sum_{p=|P(c)|-\ga}^{|P(c)|}
 {|P(c)| \choose p}|\charset(\Sigma)|^p  \}\Bigr)$
 time. 
In particular, it is polynomially solvable, if $\ga=O(1)$,
$|P(c)|=O(1)$, or 
 $|N(c)|=O(\log n \cdot |\charset(\Sigma)|)$. 
\end{thm}

\section{Deductive  Inference for Envelopes of the Exteriors of Horn Theories}
\label{sec-env}

We have considered the deduction for the interiors and exteriors  of
Horn theories. 
As mentioned before, the interiors of Horn 
theories are also Horn, while this does not hold for the exteriors. 
This means that the exteriors of Horn theories might lose beneficial properties of Horn theories. 
One of the ways to overcome such a hurdle is {\em Horn Approximation},
that is, approximating a theory by a Horn
theory~\cite{approximation}. There
are several methods for approximation, but one of the most natural ones is to approximate
a theory by its {\em Horn envelope}. 
For a theory $\Sigma$, its {\em Horn envelope} is the Horn theory $\Sigma_e$
such that $\modd(\Sigma_{e})=\ccap(\modd(\Sigma))$. 
Since Horn theories are closed under intersection, Horn envelope is the
least Horn upper bound for $\Sigma$, i.e., 
$\charset(\Sigma_e) \supseteq \charset(\Sigma)$ 
 and there exists no Horn
theory $\Sigma^*$ such that $\charset(\Sigma_e) \supsetneq \charset(\Sigma^*) \supseteq 
\charset(\Sigma)$. 
In this section, we consider the deduction for  Horn envelopes of
exteriors of Horn theories, i.e., $\sigma_{\ga}(\Sigma)_e \models c$.

\subsection{Model-Based Representations}
Let us first consider the case in which knowledge bases are represented
by characteristic sets. 

\begin{prop} \label{prop:1} 
Let $\Sigma$ be a Horn theory, and let $\ga$ be a nonnegative integer. 
Then we have 
\begin{eqnarray}
\label{eq-horn-env1}
\modd(\sigma_{\ga}(\Sigma)_e)&=&Cl_{\wedge}(\bigcup_{v\in \charset(\Sigma)}\cN_{\ga}(v)).
\end{eqnarray}
\end{prop}

\begin{proof}
By definition,  $\modd(\sigma_{\ga}(\Sigma)_e)=\ccap(\modd(\sigma_{\ga}(\Sigma))) \supseteq
 Cl_{\wedge}(\bigcup_{v\in \charset(\Sigma)}\cN_{\ga}(v))$ holds. 
For the converse direction, let $v^*$ be a model of Horn envelope of the
 $\ga$-exterior, i.e.,  $v^* \in \modd(\sigma_{\ga}(\Sigma)_e)$. 
Then $v^*$ can be represented by $v^*=\AN_{w \in W}w$ for some $W \subseteq
 \modd(\sigma_{\ga}(\Sigma))$. 
Assume that $w \in W$ is contained in $\cN_\ga(u)$ for some model
 $u$ of $\Sigma$. Since such a $u$ can be represented by $u=\AN_{z \in
 S_w}z$ for some $S_w \subseteq \charset(\Sigma)$, 
$w$ belongs to $Cl_{\wedge}(\bigcup_{v\in S_w}\cN_{\ga}(v))$. 
This, together with  $v^*=\AN_{w \in W}w$, implies that $v^*$ also belongs
 to $Cl_{\wedge}(  \bigcup_{v\in \charset(\Sigma)}\cN_{\ga}(v))$. 
\end{proof}

For a clause $c$, let $v^*$ be the unique minimal model such that
$c(v^*)=0$. 
We recall that, for a Horn theory $\Phi$,  
\begin{equation}
\label{eq-yyy}
\Phi \models c  \ \mbox{ if and only if }\ 
c(\AN_{v \in \charset(\Phi)\atop v \geq v^*} v)=1.
\end{equation}
 Therefore, Proposition \ref{prop:1} immediately implies an algorithm for
the deduction for  $\sigma_{\ga}(\Sigma)_e$ from $\charset(\Sigma)$, 
since we have $\charset(\sigma_{\ga}(\Sigma)_e) \subseteq \bigcup_{v\in \charset(\Sigma)}\cN_{\ga}(v)$. 
However, for a general $\ga$, $\bigcup_{v\in \charset(\Sigma)}\cN_{\ga}(v)$ is exponentially
larger than  $\charset(\Sigma)$, and hence this direct method is not
efficient. 
The following lemma helps developing a polynomial time algorithm. 

\begin{lem}
\label{lemma-envelope-char1}
Let $\Sigma$ be a Horn theory, let $c$ be a clause, and let $\ga$ be a nonnegative integer. 
Then $\sigma_{\ga}(\Sigma)_e \models c$ holds if and only if the
 following two conditions are satisfied. 
\begin{description}
\item[{\rm (i)}] $|\OFF(v) \cap N(c)| \geq \ga$ holds for all $v \in
	   \charset(\Sigma)$. 
\item[{\rm (ii)}]  If $S=\{v \in
	   \charset(\Sigma) \mid |\OFF(v) \cap N(c)| =\ga\} \not=\emptyset$, 
$P(c)$ is not covered with $\OFF(v)$ for models $v$
	   in $S$, i.e.,  
$P(c) \not\subseteq \bigcup_{v \in \charset(\Sigma) \atop |\OFF(v) \cap N(c)|=\ga } \OFF(v)$. 
\end{description}\end{lem}

\begin{proof}
To show the if part, let us first assume that (i) and (ii) in the lemma
 holds. 
Let $v$ be a model in $\charset(\Sigma)$ such that 
$|\OFF(v) \cap N(c)|>  \ga$.  
Then all models $w$ in $\cN_\ga(v)$ satisfy  
$\OFF(w) \cap N(c)\not=\emptyset$. 
Therefore, if all the models $v$ in $\charset(\Sigma)$ satisfy 
$|\OFF(v) \cap N(c)|>  \ga$, 
then by Proposition \ref{prop:1}, we have $\OFF(w) \cap
 N(c)\not=\emptyset$ for any model $w$ of $\sigma_{\ga}(\Sigma)_e$. 
This implies $\sigma_{\ga}(\Sigma)_e \models c$. 
Therefore, let us consider the case when $S=\{v \in
	   \charset(\Sigma) \mid |\OFF(v) \cap N(c)| =\ga\}$ is not empty.
Let $v^*$ be the unique minimal model such that $c(v^*)=0$. 
Then by Proposition \ref{prop:1}, we have 
\begin{eqnarray}
\{  v \in 
 \charset(\sigma_{\ga}(\Sigma)_e) \mid v \geq v^*\} &  & \nonumber\\ 
&&\hspace*{-3cm} \subseteq \{w
 \mid \ON(w)=\ON(v) \cup N(c) \mbox{ for some }v \in S\}. \label{eq-oo}
\end{eqnarray}
Since $P(c)$ is not covered with $\OFF(v)$ for models $v$
	   in $S$,
this, together with \raf{eq-yyy} implies $\sigma_{\ga}(\Sigma)_e \models
 c$.  

Let us next show the only-if part. 
Assume that (i) is satisfied, but (2) is not. 
Then  \raf{eq-yyy}  and \raf{eq-oo} 
imply  $\sigma_{\ga}(\Sigma)_e
 \not\models c$. 
On the other hand, if (1) is not satisfied, i.e., 
there exists a $v \in \charset(\Sigma)$ such that  $|\OFF(v) \cap N(c)| <
 \ga$, let $w^{(i)}$, $i \in P(c)$, be a model in $\cN_\ga(v)$ such  
that $\ON(w^{(i)}) \supseteq N(c)$ and $\OFF(w^{(i)}) \supseteq
 \{i\}$, and let $w^*=\bigwedge_{i\in P(c)}w^{(i)}$.   
Then we have $c(w^*)=0$ and $w^* \in \modd(\sigma_{\ga}(\Sigma)_e)$ by
 Proposition  \ref{prop:1}.  
This implies $\sigma_{\ga}(\Sigma)_e
 \not\models c$. 
\end{proof}

The lemma immediately implies the following theorem. 

\begin{thm}\label{theo:intenv:algo}
Given the characteristic set $\charset(\Sigma)$ of a Horn theory
 $\Sigma$, a clause $c$, and a nonnegative integer $\ga$, 
a deductive query $\sigma_{\ga}(\Sigma)_e \models c$ can be  
answered in  linear time.  
\end{thm}

We remark that this contrasts with Corollary \ref{cor-interior-char}. 
Namely, if we are given the characteristic set $\charset(\Sigma)$ of a
Horn theory $\Sigma$, $\sigma_{\ga}(\Sigma)_e \models c$ is
polynomially solvable, while it is {\rm co}-{\rm NP}-complete to
decide if $\sigma_{\ga}(\Sigma) \models c$.

\subsection{Formula-Based Representation}
Recall that a {\em negative} theory (i.e., a theory consisting of 
clauses with no positive literal) is Horn and the exteriors of negative
theory are also negative, and hence Horn.
This means that, for a negative theory $\Sigma$, we have
$\sigma_{\ga}(\Sigma)_e=\sigma_{\ga}(\Sigma)$.  
Therefore, we can again make use of the reduction in the proof of  Theorem
\ref{theo:int:f:g}, since the reduction uses negative theories.  

\begin{thm}\label{theo:env:ext}
Given a Horn theory $\Sigma$, a clause $c$, and a nonnegative integer $\ga$, 
it is {\rm co}-{\rm NP}-complete to decide whether $\sigma_{\ga}(\Sigma)_e
 \models c$ holds, even if $P(c)=\emptyset$.  
\end{thm}

\begin{proof}
Since the hardness is proved similarly to  Theorem
\ref{theo:int:f:g}, we show that the problem belongs to co-NP. 

Let $v$ be a model of $\sigma_{\ga}(\Sigma)_e$. 
Then $v$ can be represented by $v=\AN_{w \in W}w$ for some $W \subseteq
 \charset(\sigma_{\ga}(\Sigma))$. 
Since $\charset(\sigma_{\ga}(\Sigma)) \subseteq \bigcup_{w\in
 \charset(\Sigma)}\cN_{\ga}(w)$ by Proposition \ref{prop:1}, we have
\begin{eqnarray}
v&=&\AN_{w \in \charset(\Sigma)}\Bigl(\AN_{u \in  S_w}u\Bigr)
\label{eq-last1}
\end{eqnarray}
for some $S_w \subseteq \cN_{\ga}(w)$. We claim that there exists such a
 representation that $|S_w|\leq n$
 holds for
 all $w$'s in \raf{eq-last1}. 
 Let $w^*=\AN_{u \in  S_w}u$, and let $I=\ON(w^*) \cap \OFF(w)$  and 
$J=\OFF(w^*) \cap \ON(w)$. Then we have $w^*=\AN_{j \in J}\Bigl(
 w-e^{(j)}+\sum_{i \in I}e^{(i)} \Bigr) $, 
where $e^{(i)}$ denotes the $i$th unit model.   Since 
 $w-e^{(j)}+\sum_{i \in I}e^{(i)} \in \cN_{\ga}(w)$ for all $j \in J$, 
the claim is proved. 

Note that $\sigma_{\ga}(\Sigma)_e
 \not\models c$ if and only if there exists a model $v$ of $\sigma_{\ga}(\Sigma)_e$
such that $c(v)=0$. 
Since any model $v$ of $\sigma_{\ga}(\Sigma)_e$ can be represented by 
$v=\AN_{w \in \charset(\Sigma)}\Bigl(\AN_{u \in  S_w}u\Bigr)$ 
for some $S_w \subseteq \cN_{\ga}(w)$ with $|S_w| \leq n$ by our claim, 
the problem is in co-NP. 
\end{proof}

However, if $\ga$ or $N(c)$ is small,  the problem becomes tractable by
algorithm {\sc Deduction-Envelope-Exterior-from-Horn-Theory} (Algorithm \ref{alg:DEEHT}). 
\begin{algorithm}[hbtp]
 \caption{Deduction-Envelope-Exterior-from-Horn-Theory}
 \label{alg:DEEHT}
 \begin{description}
\setlength{\itemsep}{0.1cm}
 \item[Input:] A Horn theory
	    $\Sigma$, a clause $c$ and a nonnegative integer $\ga$. 

 \item[Output:] Yes, if $\sigma_{\ga}(\Sigma)_e \models c$; Otherwise,
		      No.
\smallskip

 \item[Step\,1.] /*  Check if there exists a model $v$ of $\Sigma$ such that   $|\OFF(v) \cap N(c)| < \ga$.
*/ \\[.1cm]
{\bf For} each $N \subseteq N(c)$ with $|N|=|N(c)|-\ga+1$ {\bf do}\\[.1cm]  
\hspace*{.7cm}
\parbox[t]{10.8cm}{Check if the theory obtained from $\Sigma$ by assigning $x_i=1$ for $i
		      \in N$ is satisfiable.} \\[.2cm]
\hspace*{.7cm}
{\bf If} so, {\bf then} output No and halt.  \\[.1cm]  
{\bf end}$\{$for$\}$

 \item[Step\,2.] /* Check if there exists a set $S=\{v \in
	   \modd(\Sigma) \mid |\OFF(v) \cap N(c)| =\ga\}$ such that 
$\bigcup_{v \in S} \OFF(v) \supseteq P(c)$.  */ \\[.1cm]
Let $J:=\emptyset$. \\[.1cm]
{\bf For} each  $N \subseteq N(c)$ with $|N|=|N(c)|-\ga$ {\bf do}\\[.1cm]  
\hspace*{0.7cm}
\parbox[t]{10.8cm}{Compute a unique minimal satisfiable model $v$ for 
the theory obtained from $\Sigma$ by assigning $x_i=1$ for $i
		      \in N$ is satisfiable. }\\[.2cm]
\hspace*{0.7cm}
Update $J := J \cup \{j \in P(c) \mid v_j =0\}$.  \\[.1cm]  
{\bf end}$\{$for$\}$\\[.1cm]
{\bf If} $J=P(c)$, {\bf then} output NO and halt. 

\item[Step\,3.] Output Yes and halt.  \qed
 \end{description}
\end{algorithm}

The algorithm is based on a necessary and sufficient condition
for $\sigma_{\ga}(\Sigma)_e \models c$, which is obtained from 
Lemma \ref{lemma-envelope-char1} by replacing all  $\charset(\Sigma)$'s
with $\modd(\Sigma)$'s. 
It is not difficult to see that such a condition holds from the proof of 
Lemma \ref{lemma-envelope-char1}.

\begin{thm}
\label{theo-env-interiro-CNF2}
Given a Horn theory
 $\Sigma$, a clause $c$, and a nonnegative integer $\ga$, 
a deductive query $\sigma_{\ga}(\Sigma)_e \models c$ can be  
answered in    $O\Bigl( \bigl({|N(c)| \choose \ga-1} +{|N(c)| \choose \ga}\bigr)
 \parallel \!\Sigma \!\parallel +|P(c)|\Bigr)$ time.
In particular, it is polynomially solvable,  if $\ga=O(1)$ or
 $|N(c)|=O(\log \parallel\!\Sigma \!\parallel)$.  
\end{thm}

\begin{proof}
The correctness of the algorithm follows from the discussion after its
 description. 
For the time complexity, it is known \cite{linear} that the satisfiability problem,
 together with computing a unique minimal model for a Horn theory, is
 possible in linear time. Since the number of the iterations of
 for-loops in Steps 2 and 3 are bounded by ${|N(c)| \choose \ga-1}$ and ${|N(c)|
 \choose \ga}$, respectively, the algorithm requires  $O\Bigl( \bigl({|N(c)| \choose \ga-1} +{|N(c)| \choose \ga}\bigr)
 \parallel\!\Sigma \!\parallel +|P(c)|\Bigr)$ time.
\end{proof}

%\section{Conclusion}

\end{document}